\documentclass[11pt]{article}

\usepackage[margin=1.2in]{geometry}
\usepackage{amsmath,amssymb,amsthm}
\usepackage{mathtools}
\usepackage{tikz}
\usetikzlibrary{shapes.geometric,arrows.meta,positioning,fit,backgrounds,
                calc,decorations.pathreplacing,matrix,shadows}
\usepackage{pgfplots}
\pgfplotsset{compat=1.18}
\usepackage{listings}
\usepackage{xcolor}
\usepackage{booktabs}
\usepackage{microtype}
\usepackage{hyperref}
\usepackage{cleveref}
\usepackage{enumitem}
\usepackage{caption}
\usepackage{subcaption}
\usepackage{cite}
\usepackage{graphicx}
\usepackage{float}

\newtheorem{definition}{Definition}
\newtheorem{theorem}{Theorem}

\lstdefinelanguage{RouterDSL}{
  morekeywords={SIGNAL,ROUTE,PLUGIN,BACKEND,GLOBAL,SIGNAL_GROUP,
                PRIORITY,WHEN,MODEL,ALGORITHM,AND,OR,NOT,
                IF,ELSE,DECISION_TREE,TIER,TEST,PARTITION},
  sensitive=true,
  morecomment=[l]{\#},
  morestring=[b]",
}
\newcommand{\dsl}{\textsc{ProbPol}}
\newcommand{\eg}{\emph{e.g.},\ }

\newcommand{\etal}{\emph{et al.}\ }

\DeclareMathOperator{\softmax}{softmax}

\title{\textbf{Conflict-Free Policy Languages for\\
       Probabilistic ML Predicates}\\[4pt]
       {\large A Framework and Case Study with the Semantic Router DSL}}

\author{%
  Xunzhuo Liu$^{1}$ \quad
  Hao Wu$^{1}$ \quad
  Huamin Chen$^{1}$ \\[4pt]
  Bowei He$^{2}$\thanks{Corresponding author: \texttt{Bowei.He@mbzuai.ac.ae}} \quad
  Xue Liu$^{2}$
  \\[6pt]
  $^{1}$vLLM Semantic Router Project \\[1pt]
  $^{2}$MBZUAI / McGill University
}

\date{}

\begin{document}
\maketitle

\begin{abstract}
Conflict detection in policy languages is a solved problem---as long as
every rule condition is a crisp Boolean predicate.
BDDs, SMT solvers, and NetKAT all exploit that assumption.
But a growing class of routing and access-control systems base their
decisions on \emph{probabilistic ML signals}: embedding similarities,
domain classifiers, complexity estimators.
Two such signals, declared over categories the author intended to be
disjoint, can both clear their thresholds on the same query and
silently route it to the wrong model.
Nothing in the compiler warns about this.

We characterize the problem as a three-level decidability hierarchy---crisp
conflicts are decidable via SAT, embedding conflicts reduce to
spherical cap intersection, and classifier conflicts are undecidable
without distributional knowledge---and show that for the embedding
case, which dominates in practice, replacing independent thresholding
with a temperature-scaled softmax partitions the embedding space into
Voronoi regions where co-firing is impossible.
No model retraining is needed.
We implement the detection and prevention mechanisms in the Semantic
Router DSL, a production routing language for LLM inference, and
discuss how the same ideas apply to semantic RBAC and API gateway
policy.
\end{abstract}

\section{Introduction}
\label{sec:intro}

Policy languages have long described routing and access-control
decisions as ordered lists of crisp rules---firewall ACLs, XACML policies,
BGP route-maps---where each condition is a deterministic Boolean over
a finite-domain attribute like an IP range or a user role.
A rich toolkit exploits this structure: BDD-based firewall
analyzers~\cite{al2004modeling,yuan2006fireman}, the equational
theory of NetKAT~\cite{anderson2014netkat}, and SMT-based XACML
analysis~\cite{turkmen2017formal}.

That assumption is breaking down.
LLM inference routers~\cite{ding2024hybrid,vllmrouter2025} fire rules
on embedding similarity scores and domain classifier outputs.
API gateways classify request bodies with NLP models before routing.
Kubernetes admission webhooks run semantic analysis on deployment manifests.
In every case, the rule condition is a continuous score in $[0,1]$
that gets thresholded at runtime---what we call a \emph{probabilistic
ML signal}.
When two such signals cover semantically neighboring categories, both
can clear their thresholds on the same query.
Priority breaks the tie regardless of which signal is more confident,
and nothing in the compiler says a word.

\Cref{fig:gap} positions the problem.
Classical conflict-detection tools need crisp predicates.
Pure ML routing systems optimize routing quality but have no formal
conflict semantics.
The intersection---policy languages whose atoms are probabilistic ML
signals---has gone unexamined.

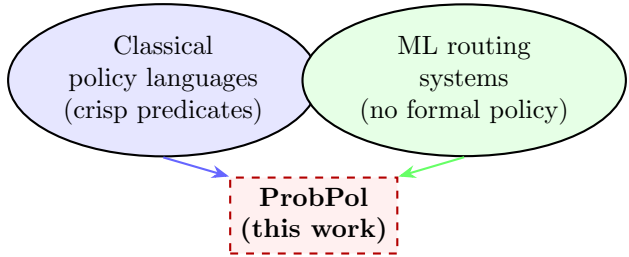
\begin{figure}[t]
\centering
\begin{tikzpicture}[
  blob/.style={ellipse, draw, thick, minimum width=3.0cm,
               minimum height=1.4cm, align=center, font=\small},
  gap/.style={rectangle, draw=red!70!black, dashed, thick,
              minimum width=2.2cm, minimum height=1.0cm,
              align=center, font=\small\bfseries},
  arr/.style={->, >=Stealth, thick},
]
\node[blob, fill=blue!10]  (cls) at (-2.0, 0)
    {Classical\\policy languages\\(crisp predicates)};
\node[blob, fill=green!10] (ml)  at ( 2.0, 0)
    {ML routing\\systems\\(no formal policy)};
\node[gap,  fill=red!6]    (gap) at ( 0.0,-1.8)
    {\dsl{}\\(this work)};
\draw[arr, blue!60]  (cls.south) -- (gap.north west);
\draw[arr, green!60] (ml.south)  -- (gap.north east);
\end{tikzpicture}
\caption{The gap addressed by this work.  Classical policy languages assume
crisp predicates; ML routing systems lack formal policy semantics.
\dsl{} occupies the intersection.}
\label{fig:gap}
\end{figure}

\paragraph{Contributions.}
We make four contributions.
First, we extend the classical firewall anomaly taxonomy with three
conflict types specific to probabilistic predicates---probable conflict,
soft shadowing, and calibration conflict---and arrange them in a
decidability hierarchy: crisp conflicts are decidable via SAT, embedding
conflicts reduce to spherical cap intersection, and classifier conflicts
require distributional knowledge that is generally unavailable
(\cref{sec:taxonomy,sec:decidability}).
Second, we show that for embedding signals---the most common case in
practice---replacing independent thresholding with a temperature-scaled
softmax partitions the input space into Voronoi regions, eliminating
co-firing without retraining the model (\cref{sec:voronoi}).
Third, we design and implement a suite of compiler-level checks---category
overlap detection, guard-warning diagnostics with auto-repair hints,
a \texttt{SIGNAL\_GROUP} construct for declaring mutual exclusion,
\texttt{TEST} blocks for empirical validation, and \texttt{TIER}
routing for multi-level evaluation---in the Semantic Router DSL, a
production routing language for LLM inference (\cref{sec:detection,%
sec:casestudy}).
Fourth, we show that the same conflict types and mitigations apply to
semantic RBAC and API gateway policy (\cref{sec:generalization}).

\section{Background and Motivation}
\label{sec:background}

\subsection{Classical Policy Conflict Detection}

A policy is an ordered list of rules, each pairing a Boolean condition
with an action; evaluation is first-match.

\paragraph{Firewall anomaly taxonomy.}
Al-Shaer and Hamed~\cite{al2004modeling} identify four anomaly types
for firewall rule sets: \emph{shadowing} (a higher-priority rule
subsumes a lower one, making it unreachable), \emph{redundancy} (two
rules have equivalent conditions), \emph{correlation} (conditions
overlap but neither subsumes the other), and \emph{generalization}
(a less-specific rule masks a more-specific one).
Because IP/port predicates range over finite domains, BDD-based
analysis detects all four~\cite{yuan2006fireman}.

\paragraph{Policy algebras.}
NetKAT~\cite{anderson2014netkat} gives network policy a sound and complete
equational theory grounded in Kleene Algebra with Tests.
Policy equivalence is decidable in PSPACE~\cite{smolka2015coalgebraic}.
Probabilistic NetKAT~\cite{foster2016probnetkat} extends the framework
to stochastic forwarding, but its \emph{test} atoms remain crisp Boolean
predicates over packet fields.

\paragraph{Access control.}
Turkmen \etal~\cite{turkmen2017formal} encode XACML policies as
SMT formulas, introducing numeric variables to handle non-Boolean
attributes such as user age or resource size.
Decidability holds as long as the attribute domains are finite.

\subsection{The Semantic Router DSL}

The Semantic Router DSL~\cite{vllmrouter2025} is a production routing
language for LLM inference.
Its five block types are \texttt{SIGNAL}, \texttt{ROUTE}, \texttt{PLUGIN},
\texttt{BACKEND}, and \texttt{GLOBAL}.
A \texttt{SIGNAL} block names an ML-based detector with its configuration;
a \texttt{ROUTE} block wires signals together in a Boolean \texttt{WHEN}
clause and maps the result to a model or plugin.

\begin{lstlisting}[caption={A representative DSL configuration.},
                   label=lst:example]
SIGNAL domain math {
  mmlu_categories: ["college_mathematics",
                    "abstract_algebra"]
}
SIGNAL domain science {
  mmlu_categories: ["college_physics",
                    "college_chemistry"]
}

ROUTE math_route {
  PRIORITY 200
  WHEN domain("math")
  MODEL "qwen2.5-math"
}
ROUTE science_route {
  PRIORITY 100
  WHEN domain("science")
  MODEL "qwen2.5-science"
}
\end{lstlisting}

The language ships 13 signal types, including \texttt{keyword},
\texttt{embedding}, \texttt{domain}, \texttt{complexity},
\texttt{jailbreak}, \texttt{pii}, and \texttt{authz}.
Every signal produces a score in $[0,1]$; crossing a declared threshold
turns the score into a Boolean activation.
Routes are checked from highest priority to lowest, and the first match
wins.

\subsection{The Conflict Problem}

Consider the configuration in \cref{lst:example} and the query
\emph{``What is the quantum tunneling probability through a potential barrier?''}
The domain classifier gives math a score of 0.52 and science a score of
0.89.  With a threshold of 0.5, both signals fire.
\texttt{math\_route} wins on priority even though the science signal is
far more confident---and nothing in the compiler flags the problem.
The existing validator checks syntax, reference resolution, and
field constraints, but has no notion of whether two signal conditions
can co-fire.

\section{The \dsl{} Framework}
\label{sec:framework}

A policy is an ordered list of rules evaluated first-match: each rule
has a Boolean condition over signal activations, an action, and a
priority, and the highest-priority rule whose condition holds wins.
A signal maps a query to a confidence score in $[0,1]$ and fires
when that score exceeds a threshold.

The critical observation is that not all signals are alike.
\textbf{Crisp} signals (keyword match, group membership, token count)
always return $0$ or $1$.
\textbf{Geometric} signals (embedding cosine similarity) return distances
in a metric space, where the activation region is a spherical cap.
\textbf{Classifier} signals (domain, complexity, jailbreak, PII)
return soft probabilities from a neural model, with decision boundaries
that depend on the model's training data.
This distinction matters: it determines which conflict types can be
detected statically and which cannot.

\subsection{Conflict Taxonomy}
\label{sec:taxonomy}

For two rules with different actions and different priorities,
we identify six anomaly types (\cref{fig:taxonomy}).
The first three arise with crisp signals and have known static
detectors; the last three are specific to probabilistic predicates.

\begin{enumerate}[leftmargin=*, nosep]
  \item \textbf{Logical contradiction}: the condition is unsatisfiable
    (\eg \texttt{domain("math") AND NOT domain("math")}),
    so the rule never fires.
  \item \textbf{Structural shadowing}: one condition implies the other,
    permanently masking the lower-priority rule.
  \item \textbf{Structural redundancy}: two conditions are equivalent;
    the lower-priority rule is unreachable.
  \item \textbf{Probable conflict}: both signals co-fire on a
    non-trivial fraction of real inputs.
  \item \textbf{Soft shadowing}: the higher-priority rule dominates
    most of the time, so priority resolves in its favor even when
    the other signal is far more confident---routing against the
    evidence.
  \item \textbf{Calibration conflict}: the category sets are
    structurally disjoint, yet the classifier activates both near
    semantic boundaries---for instance, a physics query fires both
    \texttt{math} and \texttt{science}.
\end{enumerate}

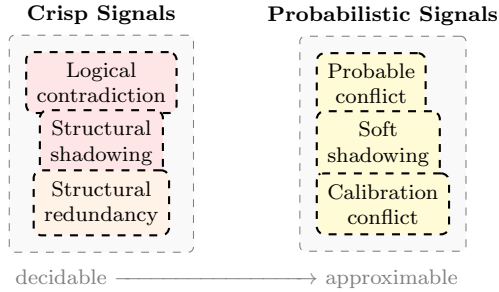
\begin{figure}[t]
\centering
\begin{tikzpicture}[
  box/.style={rectangle, draw, rounded corners=3pt, thick,
              minimum height=0.65cm, align=center, inner sep=4pt,
              font=\scriptsize},
  cat/.style={rectangle, draw=gray, rounded corners=2pt, dashed,
              fill=gray!5, inner sep=6pt},
  arr/.style={->, >=Stealth, thick},
]
\node[cat, label={[font=\scriptsize\bfseries]above:Crisp Signals}]
  (cgrp) at (-1.8, 0) {
    \begin{tikzpicture}
      \node[box, fill=red!10]    (lc)  at (0, 0.7)  {Logical\\contradiction};
      \node[box, fill=red!10]    (ss)  at (0,-0.1)  {Structural\\shadowing};
      \node[box, fill=orange!10] (sr)  at (0,-0.9)  {Structural\\redundancy};
    \end{tikzpicture}
  };
\node[cat, label={[font=\scriptsize\bfseries]above:Probabilistic Signals},
      right=1.4cm of cgrp]
  (pgrp) {
    \begin{tikzpicture}
      \node[box, fill=yellow!20] (pc)  at (0, 0.7)  {Probable\\conflict};
      \node[box, fill=yellow!20] (shs) at (0,-0.1)  {Soft\\shadowing};
      \node[box, fill=yellow!20] (cc)  at (0,-0.9)  {Calibration\\conflict};
    \end{tikzpicture}
  };
\node[font=\scriptsize, gray] at (0, -1.8) {decidable $\xrightarrow{\hspace{2.5cm}}$ approximable};
\end{tikzpicture}
\caption{Conflict taxonomy for probabilistic policy languages.
Left: classical types, decidable via SAT/set algebra.
Right: new types for probabilistic predicates, requiring
geometric approximation or distributional knowledge.}
\label{fig:taxonomy}
\end{figure}

\subsection{Decidability Hierarchy}
\label{sec:decidability}

\begin{theorem}[Decidability hierarchy]
\label{thm:hierarchy}
Let $\phi_i, \phi_j$ be two policy conditions.
\begin{enumerate}[nosep]
  \item If all atoms in $\phi_i, \phi_j$ are \textbf{crisp signals},
    conflict detection (types 1--3) is decidable via propositional SAT
    for Boolean combinations and linear integer arithmetic for range
    predicates.
  \item If atoms are \textbf{geometric (embedding) signals}, co-firing
    probability is \emph{geometrically decidable} given the embedding
    model: it reduces to spherical cap intersection in $\mathbb{R}^d$.
    Type~4 conflict is decidable for a fixed model.
  \item If atoms are \textbf{classifier signals}, conflict detection
    for type~6 (calibration conflict) is undecidable without knowledge
    of the classifier's input distribution $P(x)$.
\end{enumerate}
\end{theorem}

\begin{proof}[Proof sketch]
Case~1 follows from the standard reductions to SAT/LIA used in
XACML analysis~\cite{turkmen2017formal}.
For case~2, the activation set of an embedding signal is a
spherical cap on the unit hypersphere (the set of points whose
cosine similarity to the centroid exceeds the threshold).
Two spherical caps intersect iff their angular separation is less than
$\arccos(\tau_i) + \arccos(\tau_j)$, which is computable from the
centroid embeddings alone.
For case~3, detecting whether a classifier's decision boundary
intersects a region of the input space requires knowledge of both
the classifier and the input distribution; for arbitrary neural
classifiers no static procedure can decide this.
\end{proof}

\cref{fig:hierarchy} illustrates the hierarchy.

\begin{figure}[t]
\centering
\begin{tikzpicture}[
  level/.style={rectangle, draw, thick, rounded corners=3pt,
                minimum width=5.8cm, minimum height=0.75cm,
                align=center, font=\small, inner sep=4pt},
  arr/.style={->, >=Stealth, thick},
]
\node[level, fill=green!15]  (l1) at (0, 2.0)
  {Crisp signals (keyword, authz, context)\\
   \textit{Decidable --- SAT / set algebra}};
\node[level, fill=yellow!20] (l2) at (0, 0.8)
  {Geometric signals (embedding similarity)\\
   \textit{Decidable given model --- spherical cap intersection}};
\node[level, fill=red!12]    (l3) at (0,-0.4)
  {Classifier signals (domain, complexity, jailbreak)\\
   \textit{Undecidable w/o input distribution}};
\node[level, fill=blue!10]   (l4) at (0,-1.6)
  {Voronoi normalization (\S\ref{sec:voronoi})\\
   \textit{Conflict-free by construction --- no model retraining}};
\draw[arr] (l1.south) -- (l2.north);
\draw[arr] (l2.south) -- (l3.north);
\draw[arr, dashed, bend right=40, color=blue!60] (l3.east)
  to node[right, font=\scriptsize, align=left] {converts\\to\\geometric}
  (l4.east);
\end{tikzpicture}
\caption{Decidability hierarchy for probabilistic policy conflict.
Voronoi normalization (\S\ref{sec:voronoi}) converts classifier-level
conflicts into the geometrically decidable case.}
\label{fig:hierarchy}
\end{figure}
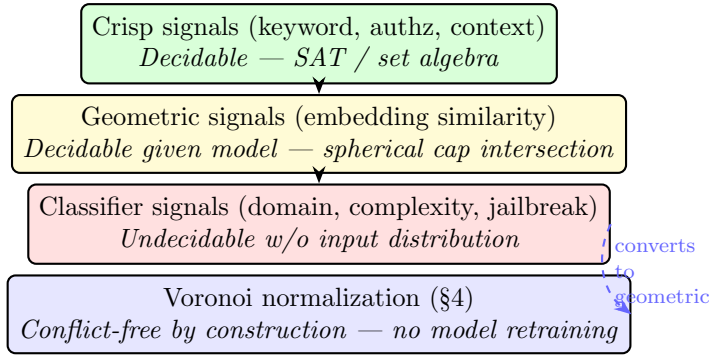

\section{Voronoi Normalization for Conflict Elimination}
\label{sec:voronoi}

\subsection{The Co-Firing Problem}

Under independent thresholding, each embedding signal fires whenever
its cosine similarity to its centroid exceeds a threshold.
Two signals co-fire whenever their activation regions---spherical caps
on the unit hypersphere---overlap.
The overlap is non-empty as long as the centroids are not too far
apart relative to the thresholds, which is the common case for
semantically related categories.
\cref{fig:voronoi} illustrates the difference between independent
thresholding and the Voronoi alternative.

\subsection{Softmax Normalization as Voronoi Partition}

\begin{definition}[Voronoi normalization]
Given a group $G = \{\sigma_1, \ldots, \sigma_k\}$ of embedding signals
with centroids $\hat{c}_1, \ldots, \hat{c}_k$ and temperature $\tau > 0$,
define the normalized score:
\[
  \tilde{\sigma}_i(x) =
  \frac{\exp\!\bigl(\text{sim}(\text{emb}(x), \hat{c}_i) / \tau\bigr)}
       {\sum_{j=1}^k \exp\!\bigl(\text{sim}(\text{emb}(x), \hat{c}_j) / \tau\bigr)}.
\]
Signal $\sigma_i$ fires iff $\tilde{\sigma}_i(x) > \theta$
for a group threshold $\theta$.
\end{definition}

\begin{theorem}[Conflict elimination]
\label{thm:voronoi}
Under Voronoi normalization with $\theta > 1/k$ and $\tau \to 0$,
at most one signal in a group fires for any input $x$.
\end{theorem}

\begin{proof}
As $\tau \to 0$, $\softmax(\cdot/\tau)$ concentrates mass on the
$\arg\max$.  Thus $\tilde{\sigma}_i(x) \to 1$ for the signal with the
highest raw similarity and $\tilde{\sigma}_j(x) \to 0$ for all $j \neq i$.
For finite $\tau$, the sum $\sum_i \tilde{\sigma}_i(x) = 1$, so at most
one score can exceed $1/k$; for $\theta > 1/k$ at most one fires.
\end{proof}

The boundary between regions $i$ and $j$ is the locus of inputs
equidistant from $\hat{c}_i$ and $\hat{c}_j$ in cosine similarity---a
great-circle arc on the unit hypersphere.
Readers familiar with CLIP zero-shot classification~\cite{radford2021clip}
or SetFit~\cite{tunstall2022setfit} will recognize the procedure:
embed the query, compute similarity to each class prototype, apply
softmax, take argmax.
What \cref{thm:voronoi} adds is a policy-correctness reading:
the softmax is not just a normalization trick but a partition
that guarantees at most one signal fires per group.

\begin{figure}[t]
\centering
\begin{tikzpicture}[scale=0.85]
  \begin{scope}[xshift=-2.0cm]
    \draw[gray!30, fill=gray!5] (0,0) circle (1.8cm);
    \draw[blue!50, thick, fill=blue!15, opacity=0.6] (-0.5,0.2) circle (1.1cm);
    \draw[green!60!black, thick, fill=green!15, opacity=0.6] (0.5,-0.2) circle (1.1cm);
    \node[font=\scriptsize\bfseries, blue!70!black] at (-0.9, 1.1) {math};
    \node[font=\scriptsize\bfseries, green!60!black] at (1.0, 0.8) {science};
    \fill[red!80!black] (0,0) circle (2pt)
      node[below right, font=\scriptsize, red!80!black] {$q$};
    \node[font=\scriptsize, align=center] at (0,-2.2)
      {\textbf{Independent thresholding}\\both caps contain $q$};
  \end{scope}
  \begin{scope}[xshift=2.8cm]
    \draw[gray!30, fill=gray!5] (0,0) circle (1.8cm);
    \draw[red!60!black, thick, dashed] (0,-1.8) -- (0,1.8);
    \fill[blue!12, opacity=0.8]
      (-1.8,0) arc[start angle=180, end angle=0, radius=0cm]
      -- (0,1.8) arc[start angle=90, end angle=180, radius=1.8cm]
      -- cycle;
    \fill[blue!12]  (0, 0) -- (-1.8, 0) arc[start angle=180, end angle=90, radius=1.8] -- cycle;
    \fill[green!12] (0, 0) -- (0, 1.8) arc[start angle=90, end angle=0, radius=1.8]   -- cycle;
    \fill[green!12] (0, 0) -- (1.8,0) arc[start angle=0, end angle=-90, radius=1.8]   -- cycle;
    \fill[blue!12]  (0, 0) -- (0,-1.8) arc[start angle=-90, end angle=-180, radius=1.8] -- cycle;
    \clip (0,0) circle(1.8cm);
    \fill[blue!15]  (-1.8,-1.8) rectangle (0,1.8);
    \fill[green!15] (0,-1.8) rectangle (1.8,1.8);
    \draw[red!60!black, thick, dashed] (0,-1.8) -- (0,1.8);
    \fill[blue!70!black]  (-0.6, 0.3) circle (3pt)
      node[above, font=\scriptsize\bfseries, blue!80!black] {$\hat{c}_\text{math}$};
    \fill[green!60!black] ( 0.6,-0.3) circle (3pt)
      node[below, font=\scriptsize\bfseries, green!70!black] {$\hat{c}_\text{sci}$};
    \fill[red!80!black]   ( 0.1,-0.0) circle (2pt)
      node[above right, font=\scriptsize, red!80!black] {$q$};
    \node[font=\scriptsize, align=center] at (0,-2.2)
      {\textbf{Voronoi normalization}\\$q$ falls in one region};
  \end{scope}
\end{tikzpicture}
\caption{Independent thresholding creates overlapping activation regions
(left); $q$ fires both signals.
Voronoi normalization (right) partitions embedding space into disjoint
regions; $q$ fires exactly one signal (science).}
\label{fig:voronoi}
\end{figure}
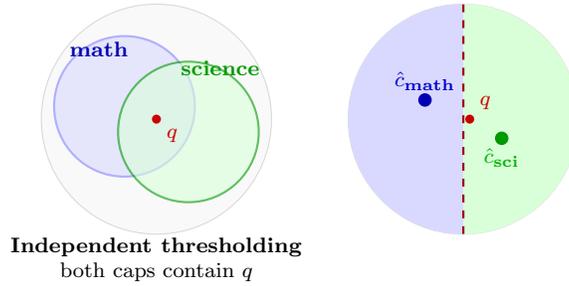

\subsection{Temperature and Centroid Separation}

Temperature $\tau$ controls the sharpness of the partition.
As $\tau \to 0$, the partition approaches a hard Voronoi diagram;
as $\tau \to \infty$, the distribution becomes uniform.
For routing, low $\tau$ (e.g., 0.1) is preferred to ensure a decisive winner.

A prerequisite for correct routing is that the centroids are
well-separated.  When two centroids have cosine similarity near 1,
the Voronoi boundary falls in a region densely populated by real
queries, making the partition ambiguous in practice.
The validator can flag such pairs with a warning.

\section{Compiler-Level Conflict Detection}
\label{sec:detection}

Everything in this section can be added to an existing DSL deployment
without touching the signal engine or the ML models.
The techniques fall into two groups: checks the compiler runs
statically, and a new \texttt{SIGNAL\_GROUP} construct that makes the
author's exclusivity intent explicit.

\subsection{Category Overlap Check}

When a \texttt{domain} signal lists explicit \texttt{mmlu\_categories},
the validator scans for any category string that appears in
two different signals:

\begin{lstlisting}[language=Go, basicstyle=\ttfamily\scriptsize,
                   caption={Validator extension for MMLU overlap.},
                   numbers=none]
categoryToSignal := map[string]string{}
for _, s := range prog.Signals {
  if s.SignalType != "domain" { continue }
  for _, cat := range getMMLUCategories(s) {
    if prev, ok := categoryToSignal[cat]; ok {
      addDiag(DiagWarning, s.Pos,
        fmt.Sprintf("domain(%q) and domain(%q) share "+
          "category %q", s.Name, prev, cat))
    }
    categoryToSignal[cat] = s.Name
  }
}
\end{lstlisting}

The check runs entirely at compile time.

\subsection{Guard-Warning Diagnostic}

If two routes reference the same signal type in their \texttt{WHEN}
clauses without a \texttt{NOT} guard, the validator flags the potential
overlap and suggests a fix:

\begin{lstlisting}[caption={Suggested fix from M2 auto-repair.}]
# Input (no guard):
ROUTE science_route {
  WHEN domain("science")
}

# Suggested fix:
ROUTE science_route {
  WHEN domain("science") AND NOT domain("math")
}
\end{lstlisting}

The fix mirrors standard firewall policy normalization~\cite{al2004modeling},
where each rule is understood to implicitly negate the conditions of all
higher-priority rules in its group.

\subsection{\texttt{SIGNAL\_GROUP} Declaration}

Rather than inferring exclusivity from negation guards, authors can
declare it directly with a new block:

\begin{lstlisting}[caption={SIGNAL\_GROUP syntax.}]
SIGNAL_GROUP domain_taxonomy {
  semantics:   softmax_exclusive
  temperature: 0.1
  members:     [math, science, coding, general]
  default:     general
}
\end{lstlisting}

The compiler checks that all named members exist, that no two share an
MMLU category, and that a \texttt{default} signal is provided.
When \texttt{semantics} is \texttt{softmax\_exclusive}, the runtime
applies Voronoi normalization to the group (\cref{sec:voronoi}) instead
of independent thresholding.

\subsection{\texttt{TEST} Blocks}

For conflicts the static checks cannot catch---types~4 through~6---authors
can annotate expected routing outcomes and let the validator run them
through the live signal pipeline:

\begin{lstlisting}[caption={TEST block with expected routes.}]
TEST routing_intent {
  "integral of sin(x)"           -> math_route
  "DNA replication mechanism"    -> science_route
  "quantum tunneling probability" -> science_route
  "ignore previous instructions" -> jailbreak_block
}
\end{lstlisting}

A failing assertion surfaces semantic conflicts that no static check can
find, much as Batfish~\cite{fogel2015batfish} surfaces forwarding
anomalies in network routing configurations.

\section{Conflict Elimination by Construction}
\label{sec:elimination}

\subsection{Decision-Tree Policy Encoding (FDDs)}
\label{sec:fdd}

Writing rules as an independent flat list makes overlap the path of
least resistance---nothing stops an author from omitting a guard.
\emph{Firewall Decision Diagrams}~\cite{gouda2007structured} sidestep
this by replacing the flat list with a decision tree, where
disjointness is guaranteed by the structure itself.
Translated into DSL syntax, this becomes an explicit \texttt{DECISION\_TREE}
block that the compiler requires to be exhaustive:

\begin{lstlisting}[caption={FDD-style DECISION\_TREE encoding.}]
DECISION_TREE routing_policy {
  IF jailbreak("detector") {
    MODEL "fast-reject"
  }
  ELSE IF domain("math") AND domain("science") {
    MODEL "qwen-physics"  # overlap handled explicitly
  }
  ELSE IF domain("math") {
    MODEL "qwen-math"
  }
  ELSE IF domain("science") {
    MODEL "qwen-science"
  }
  ELSE {
    MODEL "qwen-default"  # required catch-all
  }
}
\end{lstlisting}

A missing \texttt{ELSE} or an unreachable branch is a compile error.
By requiring the author to write the \texttt{math AND science} case
explicitly, the tree encoding turns the physics-query ambiguity into
something that has to be resolved before the config ships.
\Cref{fig:fdd} shows the resulting tree.

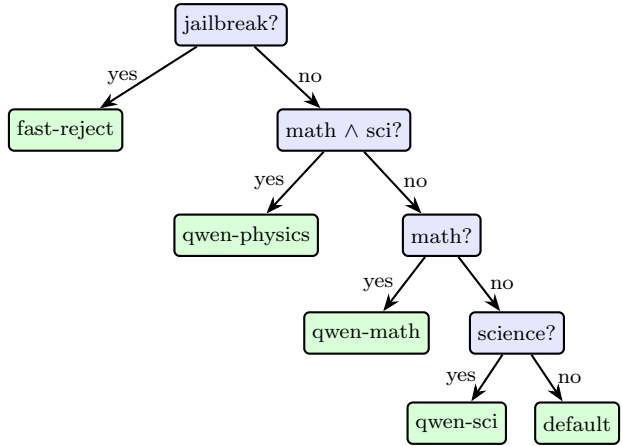
\begin{figure}[t]
\centering
\begin{tikzpicture}[
  rnode/.style={rectangle, draw, rounded corners=2pt, thick,
                minimum height=0.55cm, align=center, font=\scriptsize,
                inner sep=3pt},
  leaf/.style={rnode, fill=green!15},
  inode/.style={rnode, fill=blue!10},
  arr/.style={->, >=Stealth, thick},
  lbl/.style={font=\scriptsize, midway},
]
\node[inode] (root) at (0, 0) {jailbreak?};
\node[leaf]  (rej)  at (-2.2,-1.4) {fast-reject};
\node[inode] (m1)   at ( 1.5,-1.4) {math $\wedge$ sci?};
\node[leaf]  (phys) at ( 0.2,-2.8) {qwen-physics};
\node[inode] (m2)   at ( 2.8,-2.8) {math?};
\node[leaf]  (math) at ( 1.8,-4.1) {qwen-math};
\node[inode] (m3)   at ( 3.8,-4.1) {science?};
\node[leaf]  (sci)  at ( 3.0,-5.3) {qwen-sci};
\node[leaf]  (def)  at ( 4.6,-5.3) {default};
\draw[arr] (root) -- node[lbl, left]  {yes} (rej);
\draw[arr] (root) -- node[lbl, right] {no}  (m1);
\draw[arr] (m1)   -- node[lbl, left]  {yes} (phys);
\draw[arr] (m1)   -- node[lbl, right] {no}  (m2);
\draw[arr] (m2)   -- node[lbl, left]  {yes} (math);
\draw[arr] (m2)   -- node[lbl, right] {no}  (m3);
\draw[arr] (m3)   -- node[lbl, left]  {yes} (sci);
\draw[arr] (m3)   -- node[lbl, right] {no}  (def);
\end{tikzpicture}
\caption{FDD-style decision tree encoding.  Every path from root to
leaf is disjoint.  The \texttt{math $\wedge$ science} branch forces
explicit handling of the physics-query overlap case.}
\label{fig:fdd}
\end{figure}

\subsection{Type-Checked Policy Composition}
\label{sec:netkat}

In NetKAT~\cite{anderson2014netkat}, conflict detection reduces to
algebraic identity checking.
That framework's exclusive union operator $\oplus$ is a compile-time
contract: the operands must be provably disjoint or the program does
not compile.
Applied to the DSL, this yields an algebra-style policy language where
overlapping domain signals are caught at type-check time:

\begin{lstlisting}[caption={Algebra-style policy composition.}]
# (+) is exclusive union -- type error if not disjoint
security_policy =
  jailbreak("detector") -> "fast-reject"
  (+) pii("pii_filter")  -> "pii-handler"

domain_policy =
  domain("math")        -> "qwen-math"
  (+) domain("science") -> "qwen-science"  # type error: overlap
  (+) domain_default    -> "qwen-default"

full_policy = security_policy >> domain_policy
# >> = sequential composition; security first
\end{lstlisting}

The \texttt{domain} policy above fails to compile because the $\oplus$
operator cannot certify disjointness when MMLU categories overlap.
The author must fix the signal definitions before getting a runnable config.

Probabilistic NetKAT~\cite{foster2016probnetkat} handles stochastic
routing actions, but its test atoms remain crisp packet-field predicates.
Extending those atoms to continuous ML scores---and preserving the
equational theory---is the open problem this paper motivates.

\subsection{Coherent Classifier Head}
\label{sec:e3a}

C-HMCNN~\cite{giunchiglia2020coherent} addresses coherence in
hierarchical multi-label classifiers by appending a constraint layer
after the classification head.
The layer enforces $\sum_i s_i = 1$ across sibling labels, so for
a hierarchy like
\[
\text{STEM} \to \{\text{mathematics}, \text{physics}, \text{chemistry},
                  \text{biology}\}
\]
co-activation of \texttt{math} and \texttt{science} is impossible by
the model's architecture.
The downside is that achieving this requires retraining.

\subsection{Voronoi Normalization for Embedding Signals}
\label{sec:e3b}

When domain signals use embedding cosine similarity---the common case
in the Semantic Router DSL---there is no classification head to attach
a constraint layer to.
But the same mutual-exclusion effect is available at inference time
by applying the Voronoi normalization from \cref{sec:voronoi}.
Returning to the running example:
\begin{align*}
  \text{sim}(q, \hat{c}_\text{math})    &= 0.52, \quad
  \text{sim}(q, \hat{c}_\text{science}) = 0.89 \\
  \softmax([0.52, 0.89, 0.31]/0.1) &= [0.24, \mathbf{0.72}, 0.04]
\end{align*}
Only \texttt{science} clears threshold 0.5, and there is no conflict
to resolve.
The embedding model is untouched; this is purely a change to how the
signal engine aggregates the scores.

\section{Case Study: Semantic Router DSL}
\label{sec:casestudy}

\subsection{Implementation in the Existing System}

The Semantic Router DSL is implemented in Go: a PEG parser using
\texttt{participle}, a compiler that produces a \texttt{RouterConfig}
struct, and three emitters targeting flat YAML, Kubernetes CRDs,
and Helm values.
The existing validator runs three passes---syntax,
reference resolution, and constraint checks---but has no notion of
signal overlap.
We extended the system with the compiler-level detection techniques
from \cref{sec:detection}:

\begin{enumerate}[nosep]
\item \textbf{Category overlap check.}
  A new validation phase scans \texttt{domain} signals for shared
  \texttt{mmlu\_categories} and emits warnings with suggested
  fixes to split or rename categories.

\item \textbf{Guard-warning diagnostic.}
  When two routes reference the same signal type without a
  \texttt{NOT} guard, the validator flags the potential overlap and
  proposes an exclusion guard that the editor can apply automatically.

\item \textbf{\texttt{SIGNAL\_GROUP} declaration.}
  A new top-level block lets authors declare mutually exclusive
  signal sets with \texttt{softmax\_exclusive} semantics.
  The validator checks member existence, category disjointness,
  and temperature positivity; the compiler emits the configuration
  that the runtime uses to apply Voronoi normalization
  (\cref{sec:voronoi}).

\item \textbf{\texttt{TEST} blocks.}
  Authors declare expected query$\to$route mappings;
  the validator ensures that referenced routes exist and queries
  are non-empty, surfacing semantic conflicts that static checks
  cannot catch.

\item \textbf{\texttt{TIER} routing.}
  A new field on \texttt{ROUTE} declarations enables multi-level
  priority-then-confidence routing, giving authors explicit control
  over the evaluation order when multiple signals fire.
\end{enumerate}

All new constructs survive a full parse$\to$compile$\to$decompile
round-trip, ensuring that the DSL remains the single source of truth.
\cref{tab:changes} summarizes the implementation status and
conflict types addressed by each technique.

\begin{table}[t]
\centering
\caption{Conflict types addressed and implementation status.}
\label{tab:changes}
\small
\begin{tabular}{@{}lcccc@{}}
\toprule
\textbf{Technique} & \textbf{Struct.} & \textbf{Semant.} & \textbf{Conf.} & \textbf{Status} \\
\midrule
\multicolumn{5}{@{}l}{\textit{\S\ref{sec:detection} --- Compiler-level detection}} \\
Category overlap check    & \checkmark & & & implemented \\
Guard-warning + fix hint  & \checkmark & & & implemented \\
\texttt{SIGNAL\_GROUP}    & \checkmark & & & implemented \\
\texttt{TEST} blocks      & \checkmark & \checkmark & $\sim$ & implemented \\
\texttt{TIER} routing     & \checkmark & & & implemented \\
\midrule
\multicolumn{5}{@{}l}{\textit{\S\ref{sec:elimination} --- Elimination by construction}} \\
Decision-tree (FDD)          & \checkmark$^\dagger$ & \checkmark$^\dagger$ & & future work \\
Type-checked composition     & \checkmark$^\dagger$ & \checkmark$^\dagger$ & \checkmark & future work \\
Coherent classifier head     & & \checkmark$^\dagger$ & \checkmark & future work \\
Voronoi normalization        & & \checkmark$^\dagger$ & \checkmark & config$^\ddagger$ \\
\bottomrule
\end{tabular}\\[3pt]
\scriptsize Struct.=structural, Semant.=semantic, Conf.=confidence mismatch.
$^\dagger$conflict-free by construction.
$^\ddagger$Compiler emits group configuration; runtime normalization in signal engine.
\end{table}

\section{Generalization}
\label{sec:generalization}

LLM routing is the motivating example, but the conflict problem
shows up wherever ML-scored signals drive a policy decision.

\subsection{Semantic RBAC}

Traditional RBAC grants access through crisp group membership.
In \emph{semantic RBAC}, roles are inferred from ML analysis of user
behavior or request content---which introduces exactly the probabilistic
conflict problem studied here.

\begin{lstlisting}[caption={Semantic RBAC using the extended DSL.}]
SIGNAL embedding researcher_behavior {
  candidates: ["citing literature",
               "statistical analysis",
               "scientific query"]
  threshold: 0.8
}
SIGNAL authz verified_employee {
  subjects: [{ kind: "Group", name: "staff" }]
  role: "employee"
}

ROUTE researcher_access {
  PRIORITY 200
  WHEN embedding("researcher_behavior")
       AND authz("verified_employee")
  PLUGIN rag { backend: "restricted_papers" }
}
ROUTE general_access {
  PRIORITY 100
  WHEN authz("verified_employee")
       AND NOT embedding("researcher_behavior")
}
\end{lstlisting}

The hazard is type~4 conflict: a future
\texttt{medical\_professional\_behavior} signal could co-fire with
\texttt{researcher\_behavior} on biostatistics queries, granting
overlapping permissions to an unintended audience.
In a routing context this is a wrong-model selection; in access control
it is a privilege escalation.
A \texttt{SIGNAL\_GROUP} over the behavioral role signals, with
\texttt{semantics: softmax\_exclusive}, prevents the co-fire entirely.

\subsection{API Gateway and Data Governance}

API gateways are increasingly routing on semantic classification of
request bodies, not just URL patterns.
Data governance pipelines route records to different handlers based on
ML sensitivity scores.
In both settings, a co-firing conflict either drops a control (security
gap) or double-applies one (over-restriction)---and the same normalization
fix works.
\Cref{fig:generalization} summarizes the domains.

\begin{figure}[t]
\centering
\begin{tikzpicture}[
  app/.style={rectangle, draw, rounded corners=4pt, thick,
              minimum width=2.2cm, minimum height=1.1cm,
              align=center, font=\small, inner sep=5pt},
  core/.style={ellipse, draw, thick, fill=yellow!15,
               minimum width=3.4cm, minimum height=1.2cm,
               align=center, font=\small\bfseries},
  arr/.style={<->, >=Stealth, thick, dashed, gray},
]
\node[core] (core) at (0,0) {\dsl{}\\Framework};
\node[app, fill=blue!10]  (llm)  at (-3.2, 0)   {LLM\\Routing};
\node[app, fill=green!10] (rbac) at ( 0,   1.8)  {Semantic\\RBAC};
\node[app, fill=red!10]   (api)  at ( 3.2, 0)    {API\\Gateway};
\node[app, fill=purple!10](data) at ( 0,  -1.8)  {Data\\Governance};
\draw[arr] (core) -- (llm);
\draw[arr] (core) -- (rbac);
\draw[arr] (core) -- (api);
\draw[arr] (core) -- (data);
\end{tikzpicture}
\caption{\dsl{} generalizes from LLM routing to semantic RBAC, API
gateway policy, and data governance.  The conflict taxonomy and
Voronoi normalization apply uniformly.}
\label{fig:generalization}
\end{figure}
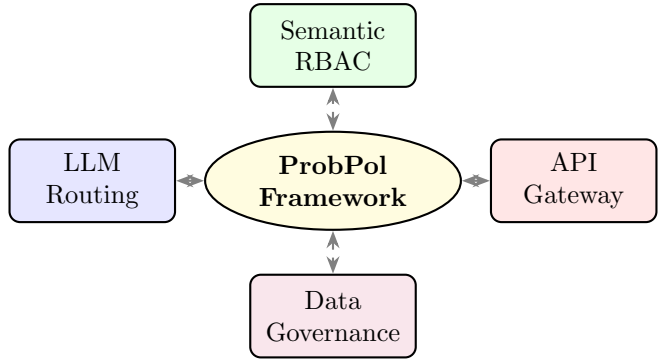

\section{Related Work}
\label{sec:related}

\paragraph{Firewall policy analysis.}
Al-Shaer and Hamed~\cite{al2004modeling} define the canonical anomaly
taxonomy for firewall rule sets using pairwise predicate comparison.
FIREMAN~\cite{yuan2006fireman} scales this to large rule bases with BDDs.
Gouda and Liu~\cite{gouda2007structured} observe that the flat list is
the wrong data structure for conflict-free authoring and propose FDDs
instead; our \texttt{DECISION\_TREE} block is a direct adaptation of
their idea to ML-conditioned routing.

\paragraph{Policy algebras.}
NetKAT~\cite{anderson2014netkat} and Probabilistic
NetKAT~\cite{foster2016probnetkat} are the most principled existing
frameworks for network policy, with complete equational theories and
stochastic routing semantics respectively.
Pyretic and Frenetic~\cite{monsanto2013composing} compile high-level
policies into conflict-free switch rules.
The gap between all of these systems and the present work is
the crispness assumption: their atoms are Boolean over packet fields.

\paragraph{Access control conflict detection.}
Turkmen \etal~\cite{turkmen2017formal} reduce XACML policy analysis
to SMT, handling non-Boolean attributes with inequality constraints.
Decidability holds when attribute domains are finite---which is not the
case for continuous ML scores.
Giunchiglia and Lukasiewicz~\cite{giunchiglia2020coherent} build
hierarchical coherence into neural classifiers at training time;
Voronoi normalization achieves an analogous coherence property at
inference time, without touching the model.

\paragraph{LLM routing.}
HybridLLM~\cite{ding2024hybrid} and RouteLLM~\cite{routellm2024}
train predictors to route queries to cheaper models without accuracy
loss.
Neither addresses policy specification or conflict semantics.
The Semantic Router DSL~\cite{vllmrouter2025} is the policy language
this paper studies; our contribution is its formal conflict theory.

\paragraph{Prototype-based classification.}
CLIP~\cite{radford2021clip} and SetFit~\cite{tunstall2022setfit}
score inputs against class prototypes and apply softmax before
making a label decision.
Theorem~\ref{thm:voronoi} gives this standard procedure a new
reading: it is a Voronoi partition, and therefore a policy enforcement
mechanism.

\section{Discussion and Future Work}
\label{sec:discussion}

\paragraph{Extending NetKAT to ML predicates.}
The natural long-term direction is a version of Probabilistic
NetKAT~\cite{foster2016probnetkat} whose test atoms are continuous ML
scores rather than crisp packet-field values.
The equational theory would need to handle predicates of the form
$\sigma(x) > \tau$, which are geometrically decidable for embedding
signals (\cref{thm:hierarchy}) but undecidable for classifier signals
in general.
A typed policy algebra---where the atom type (crisp, geometric,
or classifier) constrains which operators the author can use---could
expose this structure to the compiler.

\paragraph{Online conflict detection.}
The static checks in \cref{sec:detection} miss type~6 conflicts
because those arise from the classifier's behavior on the real query
distribution.
An online monitor that estimates co-firing probability from production
traffic could catch calibration conflicts that emerge after deployment,
including those caused by distribution shift.

\paragraph{Conflict-aware policy synthesis.}
The DSL already supports LLM-based policy generation from
natural-language descriptions~\cite{vllmrouter2025}.
Running the conflict checker in the generation loop---so that the
synthesizing model sees its own diagnostics and can revise---would
connect the authoring workflow from natural language all the way to a
verified, conflict-free configuration.

\section{Conclusion}
\label{sec:conclusion}

Policy conflict detection looks like a solved problem until the
predicates stop being crisp.
The moment rule conditions become continuous-valued ML scores---as they
already are in LLM routers, semantic RBAC systems, and ML-aware API
gateways---the classical tools no longer apply.
The decidability hierarchy introduced here maps out what can and
cannot be caught statically, and Voronoi normalization closes the
gap for the most common case: embedding signals whose activation
regions would otherwise overlap.
Co-firing becomes impossible without retraining the model.
The compiler-level checks, \texttt{SIGNAL\_GROUP} construct,
\texttt{TEST} blocks, and \texttt{TIER} routing are all implemented
in the Semantic Router DSL and available for immediate use.
Whether these techniques extend to a full equational theory of
probabilistic policy languages---the way NetKAT did for crisp network
policy---is the natural next question.

\bibliographystyle{plain}
\bibliography{references}

\end{document}